\documentclass[twoside]{article}

\usepackage[utf8]{inputenc}
\usepackage{amsfonts}
\usepackage{amsmath,amssymb}
\usepackage{nccmath}
\usepackage{mathrsfs}
\usepackage{graphicx}
\usepackage{subfigure}
\usepackage{tabularx}
\usepackage[all]{xy}
\usepackage{color}
\usepackage{multirow}
\usepackage{booktabs}
\usepackage{enumitem}
\usepackage{hyperref}
\usepackage{caption}

\newtheorem{proposition}{Proposition}
\newtheorem{lemma}[proposition]{Lemma}
\newtheorem{theorem}[proposition]{Theorem}

\newtheorem{definition}[proposition]{Definition}
\newtheorem{remark}[proposition]{Remark}

\newtheorem{example}{Example}

\newenvironment{proof}[1][Proof\ ]{\medskip\noindent{\bf #1}\ }{%
\hfill $\Box$\par\quad\par}

\usepackage{times}
\usepackage{enumitem}
\usepackage{multirow}
\usepackage{setspace} 

\def\mcl#1{\mathcal{#1}}
\def\bracket#1{\left\langle #1\right\rangle}
\def\hil{\mcl{H}}

\def\opn{\operatorname}
\def\mr{\mathrm}

\def\red#1{\textcolor{black}{#1}}

\def\eqref#1{(\ref{#1})}
\newcommand{\D}{\mathcal{D}}

\newcommand{\DD}{\mathcal{D}}
\newcommand{\RR}{\mathbb{R}}
\newcommand{\EE}{\mathbb{E}}
\newcommand{\Dbb}{\mathbb{D}}

\DeclareSymbolFont{EulerExtension}{U}{euex}{m}{n}
\DeclareMathSymbol{\euintop}{\mathop} {EulerExtension}{"52}
\DeclareMathSymbol{\euointop}{\mathop} {EulerExtension}{"48}

\begin{document}

\title{Unified generalization analysis for physics informed neural networks}

\author{Yuka Hashimoto$^{1,2}$\quad Tomoharu Iwata$^{1}$\medskip\\
{\normalsize 1. NTT, Inc.}
{\normalsize 2. RIKEN AIP}}

\date{}

\maketitle

\begin{abstract}
Physics-Informed Neural Networks (PINNs) and their variational counterparts (VPINNs) are neural networks that incorporate physical laws, making them useful for scientific problems. 
Existing generalization analyses for PINNs and VPINNs remain limited, often requiring restrictive assumptions such as stability conditions or linear ellipticity. 
In this paper, we derive generalization bounds for neural networks that involve differentiation with respect to input variables, covering PINNs and VPINNs under a unified framework. 
We apply Taylor expansion to represent nonlinear differential operators as linear operators on a high-dimensional space, enabling the use of Koopman-based analysis and showing that high-rank networks can generalize well even in settings involving differential operators. 
We also show that the nonlinearity of the differential operator exponentially enlarges the bound, highlighting its significant impact on generalization.
\end{abstract}

\section{Introduction}\label{sec:intro}
Physics-Informed Neural Networks (PINNs) are a class of machine learning methods designed to solve differential equations by incorporating physical laws directly into the training process of neural networks~\cite{raissi2019physics}.
The key idea is to train a neural network to approximate the solution of a physical system while ensuring that it satisfies the underlying equations, boundary conditions, and initial conditions.
PINNs have gained attention as part of the broader field of scientific machine learning (SciML), offering a flexible framework for solving forward and inverse problems, including parameter estimation and system identification\red{~\cite{ren2025pinnreview}}. 
They are particularly useful in areas such as fluid dynamics, solid mechanics, and diffusion processes, where traditional numerical methods may be computationally expensive or difficult to apply~\red{\cite{cai2021pinn_fluid_review,botarelli2024pinn_navier_stokes}}.
\red{Variational Physics-Informed Neural Networks (VPINNs) are a class of
PINNs that employs a weak (variational) formulation.}
They are recognized as one of the more accurate and stable approaches in this
family of methods~\cite{bhardwaj2023vpinn}.

Generalization error analysis of neural networks has been a major topic in machine learning for understanding the performance of models.
Generalization error measures the gap between the training error and the true error, and characterizes how well a model fits unseen test data.
A typical approach for deriving generalization bounds is bounding some complexity of networks~\cite{bartlett02,mohri18}.
Classically, the VC-dimension theory~\cite{harvey17,anthony09} has been applied to show a large number of parameters makes the complexity and generalization error large.
For over-parameterized regimes, norm-based bounds have been investigated~\cite{neyshabur15,bartlett17,golowich18,neyshabur18,wei19,wei20,li21,ju22,weinan22}.
These bounds do not depend on the number of parameters explicitly.
Another approach to tackle over-parameterized networks is a compression-based approach~\cite{arora18,Suzuki20}, which explain the generalization of networks by investigating how much the networks can be compressed.
These bounds get small as the ranks of the weight matrices become small.
On the other hand, Hashimoto et al.~\cite{hashimoto2024koopmanbased,hashimoto26} proposed Koopman-based bounds, which describes that high-rank networks can also generalize well.
They represent the composition structure of neural networks using linear operators called Koopman operators and derive a generalization bound by evaluating the norm of Koopman operators.
Using this approach, the determinants of weight matrices appear in the denominator of the bound, which implies high-rank weight matrices such as unitary matrices can also make the generalization error small.

The analyses in the previous paragraph are valid only for \red{standard neural networks without differential operators since the loss function depends on the derivative of neural networks, instead of the original neural networks}. 
Generalization error analysis for PINNs has also been investigated \red{to clarify how the number of collocation points and the architecture of the neural network affect the performance}.
However, the applicability of existing analyses is limited.
For example, the stability assumption~\cite{mishra2022pinn_generalization}, the linear elliptic assumption~\cite{xu2025refined} are needed for deriving the bounds.
There are no unified analysis with general settings.
In particular, generalization error analysis for VPINNs is extremely limited.
Berrone et al.~\cite{berrone2022vpinn_aposteriori} analyzed VPINNs, but the linear elliptic assumption is also needed in this case.

In this paper, we derive a generalization bound for neural networks involving differentiations of neural networks with respect to input variables, \red{including PINNs and VPINNs}.
Our setting is suitable for various types of nonlinear differential equations, \red{which can be approximated by polynomials}.
We apply the Taylor expansion and describe the nonlinear differential operator with a linear operator on a high dimensional space.
This enables us to consider the ``adjoint'' of the operator and separate the model from the differential operator.
\red{Building on this decoupled structure, we extend the Koopman-based framework to settings involving differential operators, thereby establishing new generalization bounds that reveal high-rank networks can generalize well even in the presence of differential operators.}
In addition, the nonlinearity of the differential operator exponentially makes the bound large, which implies that the nonlinearity significantly affects the generalization property.
We primarily focus on VPINNs, but the model that we consider in this paper is interpreted as an approximation of the standard PINNs.
Thus, the results are expected to also be effective for the standard PINNs.
\red{In addition, we can also apply our results to other neural networks involving differentia operators, such as Hamiltonian neural networks~\cite{greydanus2019hamiltonian} and Lagrangian neural networks~\cite{cranmer2020lagrangian}}.
\red{Practically, our bound is useful to design regularization terms.
By learning neural networks so that our bound becomes small, we can obtain high test performances.}
This paper sheds light on the generalization property of neural networks involving derivatives in a unified manner and shows the connection to the Koopman-based analysis.

\paragraph{Notations}
For a Lebesgue measure space $\mcl{X}$ and $\tilde{\mcl{X}}\subseteq \mathbb{R}^d$, $L^2(\mcl{X},\tilde{\mcl{X}})$ is the space of square integrable functions from $\mcl{X}$ to $\tilde{\mcl{X}}$.
For Banach spaces $\mcl{E}$ and $\mcl{F}$, we denote by $\mcl{B}(\mcl{F},\mcl{E})$ the space of bounded linear operators from $\mcl{F}$ to $\mcl{E}$.
If $\DD\in \mcl{B}(\mcl{F},\mcl{E})$, we denote $\DD (u)$ by $\DD u$.

\section{Preliminaries}
We review mathematical notions and existing results related to this paper.

\subsection{Fr\'{e}chet derivative and Taylor's theorem for Banach spaces}

We introduce the derivative of a map from a Banach space to a Banach space.
Let $\mcl{E},\mcl{F}$ be Banach spaces, and let $\mcl{U}\subset\mcl{F}$ be an open subset. 

\begin{definition}
A mapping $\DD:\mcl{U}\to\mcl{E}$ is said to be \emph{Fr\'{e}chet differentiable} at a point $u \in \mcl{U}$ if there exists a bounded linear operator $ \DD'\in \mathcal{B}(\mcl{F}, \mcl{E})$ such that
\begin{align*}
\lim_{\Vert h\Vert_{\mcl{F}} \to 0} \frac{\Vert \DD(u + h) - \DD(u) - \DD'(u)h\Vert_{\mcl{E}}}{\Vert h\Vert_{\mcl{F}}} = 0.
\end{align*}

The operator $\DD'(u)$ is called the \emph{Fr\'{e}chet derivative} of $\DD$ at $u$.
We denote by $\DD^{(i)}$, the $i$th Fr\'{e}chet derivative. 
If $\DD$ is Fr\'{e}chet differentiable at every point in $\mcl{U}$, then $\mcl{D}$ is said to be Fr\'{e}chet differentiable on $\mcl{U}$. 
\end{definition}

We denote by $\mcl{C}^r(\mcl{U}, \mcl{E})$ the space of functions $\DD$ where $\DD,\DD',\ldots,\DD^{(r-1)}$ are Fr\'{e}chet differentiable at each $v \in \mcl{U}$,
$\DD^{(r)}: \mcl{U} \to \mcl{B}(\mcl{F}^{\otimes r}, \mcl{E})$ is continuous. \red{Here, $\mcl{F}^{\otimes r} = \mcl{F}\otimes\cdots\otimes \mcl{F}$ ($r$ copies) and $\otimes$ is the tensor product}.

Similar to the case of the standard derivative, we have the following Taylor's theorem with the F\'{r}echet derivative.
\begin{proposition}
Let $\DD\in\mcl{C}^r(\mcl{U}, \mcl{E})$,
$v \in \mcl{U}$, and $u \in \mcl{F}$ so small that
$\{v + tu \mid 0 \le t \le 1\} \subset \mcl{U}$.
Then, we have
\begin{align*}
  \DD(v+u)
  = \sum_{i=0}^{r} \frac{\DD^{(i)}(v)\, u^{\otimes i}}{i!}
  + S(v,u)\, u^{\otimes r},
\end{align*}
where $S(v,u)
  = \int_0^1 \frac{(1-t)^{r-1}}{(r-1)!}
    (\DD^{(r)}(v+tu) - \DD^{(r)}(v))\mr{d}t$ is the remainder.
\end{proposition}

\subsection{Koopman operator and Koopman-based Rademacher complexity bound}
Koopman operator is a linear operator that represents the composition of nonlinear functions.
Since neural networks are constructed using compositions, Koopman operators play an essential role in analyzing neural networks.
Let $\mcl{X}\subseteq\mathbb{R}^{d}$ be a Lebesgue measure space.
Koopman operators are defined as follows.
\begin{definition}[Koopman operator]
Let $\tilde{\mcl{X}}\subseteq \mathbb{R}^{d_1}$ and $\mcl{X}\subseteq \mathbb{R}^{d_2}$.
The {\em Koopman operator} $K_{\sigma}$ with respect to a map $\sigma:\tilde{\mcl{X}}\to\mcl{X}$ is a linear operator from $L^2(\mcl{X},\RR^{\tilde{d}})$ to $L^2(\tilde{\mcl{X}},\RR^{\tilde{d}})$ that is defined as $K_{\sigma}h(x)=h(\sigma(x))$ for $h\in L^2(\tilde{\mcl{X}},\RR^{\tilde{d}})$.
\end{definition}

We will consider the Koopman operators with respect to activation functions.
Throughout this paper, we assume these Koopman operators are bounded.
Indeed, we have the following lemma regarding the sufficient condition of the boundedness of Koopman operators.

\begin{lemma}\label{lem:koopman_bounded}
Assume $\sigma:\tilde{\mcl{X}}\to\mcl{X}$ is bijective, $\sigma^{-1}$ is differentiable, and the Jacobian of $\sigma^{-1}$ is bounded in $\mcl{X}$.
Then, we have $\Vert K_{\sigma}\Vert\le \sup_{x\in {\mcl{X}}} \vert J\sigma^{-1}(x)\vert ^{1/2}$, where $J\sigma^{-1}$ is the Jacobian of $\sigma^{-1}$.
In particular, the Koopman operator $K_{\sigma}$ is bounded.
\end{lemma}

The following lemma is regarding the boundedness of well-known elementwise activation functions defined as $\sigma([x_1,\ldots,x_{d}])=[\tilde{\sigma}(x_1),\cdots,\tilde{\sigma}(x_{d})]$ for a map $\tilde{\sigma}:\mathbb{R}\to\mathbb{R}$.
\begin{lemma}\label{lem:sigmoid_tanh}
Let $\tilde{\mcl{X}}=[a_1,b_1]\times\cdots\times [a_d,b_d]\subseteq \mathbb{R}^d$ be a bounded rectangular domain, and let $\mcl{X}=\sigma(\tilde{\mcl{X}})$.
If $\sigma$ is the elementwise hyperbolic tangent defined as $\tilde{\sigma}(x)=\tanh(x)$, then we have $\mcl{X}\subset [-1,1]^d$ and $\Vert K_{\sigma}\Vert\le (\prod_{i=1}^d\sup_{x\in \tilde{\sigma}([a_i,b_i])}1/(1-x^2))^{1/2}$.
If $\sigma$ is the elementwise sigmoid defined as $\tilde{\sigma}(x)=1/(1+\mr{e}^{-x})$, then we have $\mcl{X}\subset [-1,1]^d$ and $\Vert K_{\sigma}\Vert\le (\prod_{i=1}^d\sup_{x\in \tilde{\sigma}([a_i,b_i])}1/(x-x^2))^{1/2}$.
\end{lemma}

Hashimoto et al.~\cite{hashimoto2024koopmanbased,hashimoto26} proposed applying Koopman operators to derive Rademacher complexity bounds of neural networks.
Let $d_l\in\mathbb{N}$, $D>0$, and $\Theta_l=\{W\in\mathbb{R}^{d_{l}\times d_{l-1}}\,\mid\, \vert\det{W_l^*W_l}\vert^{-1/4}\le D\}$. 
\red{Here, $^*$ represents the adjoint.}
Let $\mcl{X}_0\subset\mathbb{R}^{d_0}$ be the input space, $W_l\sigma_{l-1}(W_{l-1}\cdots W_2\sigma_1(W_1\mcl{X}_0))\subseteq \tilde{\mcl{X}_l}\subseteq \mathbb{R}^{d_l}$, and $\sigma_{l}(W_l\cdots W_2\sigma_1(W_1\mcl{X}_0))\subseteq \mcl{X}_l\subseteq \mathbb{R}^{d_l}$ that satisfy $\mu_{\mathbb{R}^{d_l}}(\mcl{X}_l)>0$ and $\mu_{\mathbb{R}^{d_l}}(\tilde{\mcl{X}}_l)>0$.
Here, $\mu_{\mathbb{R}^{d_l}}$ is the Lebesgue measure on $\RR^{d_l}$ and $\sigma_l:\tilde{\mcl{X}}_l\to \mcl{X}_l$ is the activation function at layer $l$.
We assume the Koopman operator $K_{\sigma_l}$ is bounded.
Let $v\in L^2(\tilde{\mcl{X}}_L, \mcl{X}_L)$ be the final nonlinear transformation.
Consider the neural network
\begin{align}
u_{\theta}(x)=v(W_L\sigma_{L-1}(W_{L-1}\sigma_{L-2}(\cdots \sigma_1(W_1x))))\label{eq:nn}
\end{align}
with learnable parameters $\theta=(W_1,\ldots,W_L)$.

Let $\alpha(h)=({\int_{W_l\mcl{X}_{l-1}} |h(x)|^2 d\mu_{\mcl{R}(W_l)}(x)}/{\int_{\tilde{\mcl{X}}_l}\vert h(x)\vert^2\mr{d}\mu_{\mathbb{R}^{d_l}}(x)})^{1/2}$ for $h\in L^2(\tilde{\mcl{X}}_l,\mcl{X}_L)$.
Let $f_l=v\circ W_L\circ \sigma_{L-1}\circ \cdots \circ W_{l+1}\circ \sigma_l$.
They derived the following lemma regarding the $L^2$ norm of $u_{\theta}$.
\begin{lemma}\label{lem:norm_u_theta}
We have 
\begin{align*}
\Vert u_{\theta}\Vert\le \frac{\Vert v\Vert\prod_{l=1}^{L-1}\Vert K_{\sigma_l}\Vert \alpha(f_l)}{\prod_{l=1}^L\vert\det W_l^*W_l\vert^{1/4}}.
\end{align*}
\end{lemma}

Let $p_{x}\in L^2(\mcl{X}_0,\RR^{d_L})$ for $x_n\in \mcl{X}_0$, and consider a regularized model $V_{\theta}(x)=\int_{\mcl{X}_0}u_{\theta}(y)\cdot p_{x}(y)\mr{d}y$.
Let $D>0$ and $\mcl{V}_{\Theta}=\{V_{\theta}\,\mid\,\theta\in\Theta\}$, where $\Theta=\{(W_1,\ldots,W_L)\,\mid\,W_l\in\Theta_l\}$.
Using Lemma~\ref{lem:norm_u_theta}, they derive the following bound.
\begin{proposition}\label{prop:rademacher_existing}
Assume there exists $F>0$ such that $\Vert p_{x_n}\Vert\le F$.
Then, the empirical Rademacher complexity of $\mcl{V}_{\Theta}$ and samples $x_1,\ldots,x_N\in\mcl{X}_0$ is bounded by
\begin{align}
\sup_{\theta\in\Theta}\frac{F \Vert v\Vert\prod_{l=1}^{L-1}\Vert K_{\sigma_l}\Vert \alpha(f_l)}{\sqrt{N}\prod_{l=1}^L\vert\det W_l^*W_l\vert^{1/4}}.\label{eq:bound_injective}
\end{align}
\end{proposition}
The right hand side of \eqref{eq:bound_injective} includes the determinants of the weight matrices.
Thus, the bounds can become smaller even if $W_l$ has large singular values.
This result with Koopman operators describes why networks with high-rank weight matrices generalize well.

\section{Problem setting}
We consider a model with differentiations of a neural network with respect to inputs.
Typical examples are PINNs, Hamiltonian neural networks, and Lagrangian neural networks.
Let $d,\tilde{d}\in\mathbb{N}$, $\mcl{X} \subset \mathbb{R}^d$ be a bounded domain, $C > 0$, and $\mathcal{F} := \{ u \in H^s(\mcl{X}, \RR^{\tilde{d}}) \,\mid\, \|u\|_{H^s(\mcl{X}, \RR^{\tilde{d}})} \le C \}$.
\red{Here, $H^s(\mcl{X},\RR^{\tilde{d}})$ is the Sobolev space of order $s$ from $\mcl{X}$ to $\RR^{\tilde{d}}$.}
Let $\mathcal{D} :  \mcl{F}\to L^2(\mcl{X}, \RR^{\tilde{d}})$ be a differential operator of order $s$.
Note that $\mcl{D}$ can be nonlinear.
Let $\Theta$ be a parameter space and consider a neural network $u_{\theta}\in \mcl{F}$ parametrized by $\theta\in\Theta$.
We denote by $\mcl{U}_{\Theta}$ the class of neural networks parametrized by $\Theta$.
Consider the differentiation of the neural network $\mcl{D}(u_{\theta})$ and its variational form:
\begin{align*}
V_{\theta}(x_n)=\int_\mcl{X} \D(u_{\theta})(y)\cdot p_{x}(y)\mr{d}y=\langle p_{x_n}, \mathcal{D}(u_{\theta}) \rangle_{L^2(\mcl{X})}\quad (n=1,\ldots,N),
\end{align*}
where $x_1, \ldots, x_N\in\mcl{X}$ are collocation points and for $x_n$, $p_{x_n} \in H^s(\mcl{X},\RR^{\tilde{d}})$ is a \red{nonzero} test function satisfying
$p_{x_n}^{(r)}(x) = 0$ for $x \in \partial\mcl{X}$, $n=1,\ldots,N$, $r=1,\ldots,s$ and there exists $B>0$ such that for $n=1,\ldots,N$, $\Vert p_{x_n}\Vert\le B$.
We denote $\mcl{V}_{\Theta}=\{V_{\theta}\,\mid\,u_{\theta}\in\mcl{U}_{\Theta}\}$, the class of regularized models. 
Consider a $L$-Lipschitz continuous loss function $\mcl{L}_n:\RR\to\RR_+$ for $n=1,\ldots,N$ and minimize the mean loss:
\begin{align}
\frac1N\sum_{n=1}^N \mcl{L}_n\bigg(\red{V_{\theta}(x_n)}\bigg).\label{eq:vpinn_res}
\end{align}

Typical examples of the test function $q_{x_n}$ and the loss function $\mcl{L}_n$ are as follows.
\begin{example}\label{ex:bump_function}
Let $q_{x_n}(y)=\mr{e}^{-1/(1-c^2\Vert y\Vert^2)}$ for $\Vert y\Vert\le 1/c$ and $q_{x_n}(y)=0$ for $\Vert y\Vert >1/c$.
Let $\tilde{d}=1$ and $p_{x_n}(y)=q_n(y)/(\int_{\mcl{X}}q_n(z)\mr{d}z$).
Then, $p_{x_n}$ is infinitely differentiable.
Consider a PDE $\mcl{D}(u)=f$ for $f\in L^2(\mcl{X},\RR^{\tilde{d}})$.
Let $\mcl{L}_n(z)=\vert z-\int_{\mcl{X}}f(y)p_{x_n}(y)\mr{d}y\vert^2$.
Assume for some $D>0$ and $E>0$, $\vert \int_\mcl{X} \D(u_{\theta})(y)p_{x_n}(y)\mr{d}y\vert\le D$ and $\vert \int_{\mcl{X}}f(y)p_{x_n}(y)\mr{d}y\vert\le E$.
Then, $\mcl{L}_n$ is $2(D+E)$-Lipchitz.
Note that since $p_{x_n}$ goes to the delta function as $c\to\infty$, we have 
\begin{align}
\lim_{c\to \infty}\int_\mcl{X} \D(u_{\theta})(y)p_{x_n}(y)\mr{d}y=\D(u_{\theta})(x_n).\label{eq:vpinn_pinn}
\end{align}
In this case, Eq.~\eqref{eq:vpinn_res} is the loss for the \red{variational form such as VPINNs, but as $c\to \infty$, it becomes the loss for the strong form, such as PINNs for the interior collocation point $x_n$}. 
\end{example}

\section{Generalization bounds}
Our goal is to derive Koopman-based generalization bounds of the \red{variational form} $V_{\theta}$.
Since we have the differential operator $\mcl{D}$ in $V_{\theta}$ and \red{the loss function depends on the derivative of the neural network, not the neural network itself}, the derivation is not straightforward.
This is the main difference between the existing setting in Proposition~\ref{prop:rademacher_existing} and our setting.
To tackle this situation, we first consider a simple case with linear differential operator $\DD$.
In this case, we can consider the adjoint of $\DD$.
We then consider nonlinear cases.
Applying the Taylor expansion, we can regard $\DD$ as a linear operator on a higher dimensional space, which enables us to consider a similar notion of the adjoint of $\DD$.

Since the loss function $\mcl{L}_n$ is Lipchitz continuous, we can derive a generalization bound by deriving a Rademacher complexity bound for the function class $\{\int_\mcl{X} \D(u_{\theta})(y)\cdot p_{x_n}(y)\mr{d}y\,\mid\,u_{\theta}\in \mcl{U}_{\Theta}\}$~\cite[Theorem 3.5]{mohri18}.
Let $\Omega$ be a probability space equipped with a probability measure $P$.
Let  $\epsilon_1,\ldots,\epsilon_N:\Omega\to\mathbb{C}$ be i.i.d. Rademacher variables (random variables following the uniform distribution on $\{-1,1\}$).
For a measurable function $\epsilon:\Omega\to\mathbb{C}$, we denote by $\mr{E}[\epsilon]$ the integral $\int_{\Omega}\epsilon(\omega)\mr{d}P(\omega)$.
The empirical Rademacher complexity $\hat{R}(\mcl{U},x_1,\ldots,x_N)$ of a function class $\mcl{U}$ is defined as $\hat{R}(\mcl{U},x_1,\ldots,x_N)=\mr{E}[\sup_{u\in \mcl{U}} \sum_{n=1}^N u(x_n) \epsilon_n]/N$.

\subsection{Linear case}
Assume $\mathcal{D}$ is \red{bounded} and linear.
Then, by the Cauchy Schwarz inequality, we have
\begin{align*}
  \sum_{n=1}^{N} \langle p_{x_n}, \mathcal{D}(u_{\theta}) \rangle
  = \bigg\langle \sum_{n=1}^{N} \mathcal{D}^* p_{x_n}, u_{\theta} \bigg\rangle
  \le \|u_{\theta}\|\,\bigg\|\sum_{n=1}^{N} \mathcal{D}^* p_{x_n}\bigg\|,
\end{align*}
where $^*$ means the adjoint.
This transformation enables us to force the action of $\D$ to $p_{x_n}$, instead of $u_{\theta}$, and to apply the same analysis as the neural network $u_{\theta}$ without $\D$. 
Indeed, we have the following theorem.
The proof is documented in Appendix~\ref{ap:proofs}

\begin{theorem}\label{thm:linear}
Let $\Dbb_{x_n}=\D^*p_{x_n}$.
Assume there exists $F>0$ such that for any $n=1,\ldots,N$, $\red{\Vert \Dbb_{x_n}\Vert} \le F$.
Then, we have 
\begin{align*}
\hat{R}(\mcl{V}_{\Theta},x_1,\ldots,x_N)\le \frac{F}{\sqrt{N}}\sup_{u_{\theta}\in\mcl{U}_{\Theta}}
      \|u_{\theta}\|. 
\end{align*}
\end{theorem}

\subsection{Polynomial case}\label{subsec:poly}
When $\DD$ is nonlinear, we do not have $\DD^*$.
However, if $\DD$ is approximated by a polynomial, then we can construct a linear operator on a higher dimensional space.
Then, we have the adjoint of the linear operator.

Let $\tilde{\DD}\in\mcl{C}^r(\mcl{F},L^2(\mcl{X},\mathbb{R}^d))$ be a differential operator whose remainder $S(0,u_{\theta})$ of $r$-order Taylor approximation at $0$ is bounded as $\Vert S(0,u_{\theta})\Vert\le \epsilon$ for some $\epsilon>0$.
Assume $\DD$ is in the following form:
\begin{align}
\DD(u_{\theta})=g\circ u_{\theta} + \tilde{\DD}(u_{\theta})\label{eq:D_poly}
\end{align}
with $g:\RR^{\tilde{d}}\to \RR^{\tilde{d}}$.
\red{In many real-world physical systems, such as fluid flows and diffusion-dominated problems}, terms in $\DD$ with differentiations are polynomials and non-polynomial terms only depend on the solution $u$, not the derivatives of $u$.
In this case, we can regard $\DD$ as a linear operator on a higher dimensional space and we have
\begin{align*}
  &\sum_{n=1}^{N} \langle p_{x_n}, \mathcal{D}(u_{\theta}) \rangle
  =\sum_{n=1}^{N}\bigg(\bracket{p_{x_n},g\circ u_{\theta}}+\sum_{i=0}^{r}\frac{1}{i!}
       \langle p_{x_n},\tilde{\DD}^{(i)}(0)u_{\theta}^{\otimes i}
       \rangle
     + \langle p_{x_n},S(0,u_{\theta})u_{\theta}^{\otimes r}
       \rangle\bigg)\\
    &\qquad=\bigg\langle \sum_{n=1}^{N}\Dbb_{x_n},
      \bigoplus_{i=0}^{r} u_{\theta}^{\otimes i} \oplus u_{\theta}^{\otimes r}\oplus (g\circ u_{\theta})
    \bigg\rangle
  \le \bigg\|\bigoplus_{i=0}^{r} u_{\theta}^{\otimes i} \oplus u_{\theta}^{\otimes r}\oplus (g\circ u_{\theta})\bigg\|\,\bigg\|\sum_{n=1}^{N} \Dbb_{x_n}\bigg\|,
\end{align*}
where $\Dbb_{x_n}=\bigoplus_{i=0}^{r} \frac{1}{i!}\DD^{(i)}(0)^{*}p_{x_n}
    \oplus S(0,u_{\theta})^{*}p_{x_n}\oplus p_{x_n}$.
Using this perspective, we have the following theorem in the same manner as Theorem~\ref{thm:linear}.
\begin{theorem}\label{thm:poly}
Let $\Dbb_{x_n}=\bigoplus_{i=0}^{r} \frac{1}{i!}\DD^{(i)}(0)^{*}p_{x_n}
    \oplus S(0,u_{\theta})^{*}p_{x_n}\oplus p_{x_n}$.
\red{Assume there exists $\epsilon>0$ such that $\Vert S(0,u_{\theta})\Vert\le \epsilon$.}
Assume, in addition, there exists $F>0$ such that for any $n=1,\ldots,N$, $\Vert\Dbb_{x_n}\Vert\le F$.
Then, we have 
\begin{align*}
\hat{R}(\mcl{V}_{\Theta},x_1,\ldots,x_N)&\le \frac{F}{\sqrt{N}}\sup_{u_{\theta}\in\mcl{U}_{\Theta}}
      \bigg\|\bigoplus_{i=0}^{r} u_{\theta}^{\otimes i} \oplus u_{\theta}^{\otimes r}\oplus (g\circ u_{\theta})\bigg\|\\
&=\frac{F}{\sqrt{N}}\sup_{u_{\theta}\in\mcl{U}_{\Theta}}
      \bigg(\sum_{i=0}^r\| u_{\theta}\|^{2i} +\| u_{\theta}\|^{2r}+ \|g\circ u_{\theta}\|^2\bigg)^{1/2}.
\end{align*}
\end{theorem}
The order $r$ represents the nonlinearlity of $\DD$.
Thus, Theorem~\ref{thm:poly} shows that as the nonlinearlity of $\DD$ grows, the Rademacher complexity becomes large.

\begin{remark}\label{rmk:Dx_bound}
Since we have 
\begin{align*}
\bracket{\Dbb_{x_n},\Dbb_{x_n}}&=\sum_{i=0}^r\bigg\Vert \frac{1}{i!}\DD^{(i)}(0)^{*}p_{x_n}\bigg\Vert^2 + \Vert S(0,u_{\theta})^{*}p_{x_n}\Vert^2 + \Vert p_{x_n}\Vert^2 \\
&\le \sum_{i=0}^r\bigg\Vert \frac{1}{i!}\DD^{(i)}(0)^{*}p_{x_n}\bigg\Vert^2 + (\epsilon^2 + 1)\Vert p_{x_n}\Vert^2, 
\end{align*}
the constant $F$ does not depend on $\theta$.
\end{remark}

\begin{example}
Let $d=\tilde{d}$. Consider the nonlinear convection operator
\begin{align*}
  \DD(u) = (u\cdot\nabla)u
        = \sum_{k=1}^{d} u_k\,\partial_k u_j
        = [1,\ldots,1]
          \begin{bmatrix}\partial_{t_1} & & \\ & \ddots &\\
          & & \partial_{t_d}\end{bmatrix}
          \begin{bmatrix}
            \tilde{u}_{11} & \cdots& \tilde{u}_{1d} \\
            \vdots & \ddots &  \vdots\\
            \tilde{u}_{1d} & \cdots & \tilde{u}_{dd}
          \end{bmatrix},
\end{align*}
where $\tilde{u}_{ij}(s,t) = u_i(s)\,u_j(t) = (u_i\otimes u_j)(s,t)$
for $s=[s_1,\ldots,s_d],[t_1,\ldots,t_d]\in\RR^{d}$. 
Note that this operator is the second order polynomial with respect to $u$.
This operator appears, for example, in the Navier--Stokes equation.

\noindent The adjoint $\DD''(0)^{*}$ acts on $v$ as
\begin{align*}
  \DD''(0)^{*}v
  &=
  \begin{bmatrix}\partial_{t_1}^{*} & & \\ & \ddots & \\ & & \partial^*_{t_d}\end{bmatrix}
  \begin{bmatrix}1\\\vdots\\ 1\end{bmatrix}
  [v_1 \;\cdots\; v_d]
  =
  \begin{bmatrix}
    \partial_1^{*}v_1 & \cdots & \partial_1^{*}v_d \\
    \vdots & \ddots & \vdots\\
    \partial_d^{*}v_1 & \cdots & \partial_d^{*}v_d 
  \end{bmatrix}
  =
  -\begin{bmatrix}
    \partial_1v_1 & \cdots & \partial_1v_d \\
    \vdots & \ddots & \vdots\\
    \partial_dv_1 & \cdots & \partial_dv_d 
  \end{bmatrix}.
\end{align*}

Therefore, we have
\begin{align*}
  \|\DD''(0)^{*}v\|^2 = \sum_{i,j=1}^{d}\|\partial_i v_j\|^2,
\end{align*}
which shows that $\|\DD''(0)^{*}p_{x_n}\|$ does not depend on $x_n$ if the norm of the derivative of $p_{x_n}$ does not depend on $x_n$.
For example, if $p_{x_n}$ is shift-invariant, i.e., $p_{x_n}(y)=\tilde{p}(y-x_n)$ for some $\tilde{p}$, then the norm of the derivative of $p_{x_n}$ does not depend on $x_n$.
As a result, by Remark~\ref{rmk:Dx_bound}, the assumption $\Vert \DD_{x_n}\Vert<F$ in Theorem~\ref{thm:poly} is satisfied. 
Note that since $\DD$ is the second-order polynomial, the order $r$ in Theorem~\ref{thm:poly} is $r=2$ in this case.
\end{example}

\subsection{Nonlinear transformation of a linear differential operator}
Another important nonlinear form is the nonlinear transformation of a linear differential operator.
Let $\DD:\mcl{F}\to L^2(\mcl{X},\RR^{\tilde{d}})$ be linear.
Assume $\DD$ is in the following form:
\begin{align}
\DD(u_{\theta})=g\circ \tilde{\DD}u_{\theta}\label{eq:D_nonlinear_liner}
\end{align}
with $\tilde{\DD}:\mcl{F}\to L^2(\Omega,\RR^{\hat{d}})$ being bounded linear and $g:\RR^{\hat{d}}\to\RR^{\tilde{d}}$.
Typical examples are the Hamilton--Jacobi equations~\cite{CrandallLions1983} and the Monge--Amp\`{e}re equations~\cite{Gong2020}.
Assume there exists $\epsilon>0$, and assume $g\in \mcl{C}^r(\RR^{\hat{d}},\RR^{\tilde{d}})$ whose remainder $S(0,\tilde{\DD}u_{\theta}(x))$ of the $r$-order Taylor approximation at $0$ is bounded as $\Vert S(0,\tilde{\DD}u_{\theta}(x))\Vert\le \epsilon$ for any $x\in \mcl{X}$.
In this case, we can approximate $\DD$ as a polynomial of $\tilde{\DD}u_{\theta}$, which is represented as a linear operator on a higher dimensional space.
We have
\begin{align*}
  \sum_{n=1}^{N} \langle p_{x_n}, \mathcal{D}(u_{\theta}) \rangle
  &=\sum_{n=1}^{N}\bigg(\sum_{i=0}^{r}\frac{1}{i!}
       \langle p_{x_n},g^{(i)}(0)(\tilde{\DD}u_{\theta}(\cdot))^{\otimes i}
       \rangle
     + \langle p_{x_n},S(0,\tilde{\DD}u_{\theta}(\cdot))(\tilde{\DD}u_{\theta}(\cdot))^{\otimes r}
       \rangle\bigg)\\
    &=\sum_{n=1}^{N}\bigg(\sum_{i=0}^{r}
       \frac{1}{i!}\langle p_{x_n},g^{(i)}(0)P^i\tilde{\DD}^{\otimes i}u_{\theta}^{\otimes i}
       \rangle
     + \langle S(0,\tilde{\DD}u_{\theta}(\cdot))^*p_{x_n}(\cdot),P^r\tilde{\DD}^{\otimes r}u_{\theta}^{\otimes r}
       \rangle\bigg)\\
   &=\bigg\langle \sum_{n=1}^{N}\Dbb_{x_n},\bigoplus_{i=0}^{r}u_{\theta}^{\otimes i}\oplus u_{\theta}^{\otimes r}\bigg\rangle
  \le \bigg\|\bigoplus_{i=0}^{r} u_{\theta}^{\otimes i} \oplus u_{\theta}^{\otimes r}\bigg\|\,\bigg\|\sum_{n=1}^{N} \Dbb_{x_n}\bigg\|, 
\end{align*}
where $\Dbb_{x_n}=\bigotimes_{i=0}^r1/(i!)(\tilde{\DD}^{\otimes i})^*(P^i)^*g^{(i)}(0)^*p_{x_n}\oplus (\tilde{\DD}^{\otimes r})^*(P^r)^*S(0,\tilde{\DD}u_{\theta}(\cdot))^*p_{x_n}(\cdot)$, $P^i:L^2(\mcl{X},\RR^{\tilde{d}})^{\otimes i}\to L^2(\mcl{X},(\RR^{\tilde{d}})^{\otimes i})$ is defined as $P^iu^{\otimes i}(x)=u^{\otimes i}(x,\ldots,x)=u(x)^{\otimes i}$, and for functions $u$ and $v$ of $x$, $u(\cdot)v(\cdot)$ means $x\mapsto u(x)v(x)$. 
In the same manner as Theorem~\ref{thm:linear}, we have the following theorem.
\begin{theorem}\label{thm:nonlinear-linear}
Let $\Dbb_{x_n}=\bigotimes_{i=0}^r1/(i!)(\tilde{\DD}^{\otimes i})^*(P^i)^*g^{(i)}(0)^*p_{x_n}\oplus (\tilde{\DD}^{\otimes r})^*(P^r)^*S(0,\tilde{\DD}u_{\theta}(\cdot))^*p_{x_n}(\cdot)$.
\red{Assume there exists $\epsilon>0$ such that $\Vert S(0,\tilde{\mcl{D}}u_{\theta})\Vert\le \epsilon$.}
Assume, in addition, there exists $F>0$ such that for any $n=1,\ldots,N$, $\Vert \Dbb_{x_n}\Vert\le F$.
Then, we have 
\begin{align*}
\hat{R}(\mcl{V}_{\Theta},x_1,\ldots,x_N)&\le \frac{F}{\sqrt{N}}\sup_{u_{\theta}\in\mcl{U}_{\Theta}}
      \bigg\|\bigoplus_{i=0}^{r} u_{\theta}^{\otimes i} \oplus u_{\theta}^{\otimes r}\bigg\|\\
&=\frac{F}{\sqrt{N}}\sup_{u_{\theta}\in\mcl{U}_{\Theta}}
      \bigg(\sum_{i=0}^r\| u_{\theta}\|^{2i} +\| u_{\theta}\|^{2r}\bigg)^{1/2}.
\end{align*}
\end{theorem}
The order $r$ represents the nonlinearity of $g$ in this case.
Thus, Theorem~\ref{thm:nonlinear-linear} shows that as the nonlinearity of $g$ grows, the Rademacher complexity becomes large.
Note that in the same manner as Remark~\ref{rmk:Dx_bound}, the factor $F$ does not depend on $\theta$.


\subsection{Application of Koopman-based bounds}\label{subsec:koopman}
Since Theorems~\ref{thm:linear}, \ref{thm:poly}, and \ref{thm:nonlinear-linear} includes the term $\Vert u_{\theta}\Vert$, we can apply Lemma~\ref{lem:norm_u_theta} to understand the relationship between the Rademacher complexity and the weight matrices.
We set $\mcl{X}_0=\mcl{X}$, $d_L=\tilde{d}$.
Consider the neural network~\eqref{eq:nn}.
Then, we obtain the following bound.

Let $\alpha(h)=({\int_{W_l\mcl{X}_{l-1}} |h(x)|^2 d\mu_{\mcl{R}(W_l)}(x)}/{\int_{\tilde{\mcl{X}}_l}\vert h(x)\vert^2\mr{d}\mu_{\mathbb{R}^{d_l}}(x)})^{1/2}$ for $h\in\tilde{\hil}_l$.
Let $D>0$, $\Theta=\{(W_1,\ldots,W_L)\,\mid\,\vert\det{W_1^*W_1}\vert^{-1/4},\ldots,\vert\det{W_L^*W_L}\vert^{-1/4}\le D\}$.
\begin{theorem}\label{thm:linear_koopman}
Assume $\DD$ is linear.
Let $\Dbb_{x_n}=\D^*p_{x_n}$.
Assume there exists $F>0$ such that for any $n=1,\ldots,N$, $\Vert\Dbb_{x_n}\Vert\le F$.
Then, we have 
\begin{align*}
\hat{R}(\mcl{V}_{\Theta},x_1,\ldots,x_N)\le \frac{F}{\sqrt{N}}\sup_{\theta\in\Theta}\frac{\Vert v\Vert\prod_{l=1}^{L-1}\Vert K_{\sigma_l}\Vert \alpha(f_l)}{\prod_{l=1}^L\vert\det W_l^*W_l\vert^{1/4}}. 
\end{align*}
\end{theorem}

\begin{theorem}\label{thm:poly_koopman}
Assume $\DD$ is in the form~\eqref{eq:D_poly}.
Let $\Dbb_{x_n}=\bigoplus_{i=0}^{r} \frac{1}{i!}\DD^{(i)}(0)^{*}p_{x_n}
    \oplus S(0,u_{\theta})^{*}p_{x_n}\oplus p_{x_n}$.
Assume there exists $\epsilon>0$ such that $\Vert S(0,u_{\theta})\Vert\le \epsilon$.
Assume, in addition, there exists $F>0$ such that for any $n=1,\ldots,N$, $\Vert\Dbb_{x_n}\Vert\le F$.
Then, we have 
\begin{align*}
&\hat{R}(\mcl{V}_{\Theta},x_1,\ldots,x_N)\\
&\le\frac{F}{\sqrt{N}}\sup_{\theta\in{\Theta}}
      \bigg(\sum_{i=0}^r\bigg(\frac{\Vert v\Vert\prod_{l=1}^{L-1}\Vert K_{\sigma_l}\Vert \alpha(f_l)}{\prod_{l=1}^L\vert\det W_l^*W_l\vert^{1/4}}\bigg)^{2i} +\bigg(\frac{\Vert v\Vert\prod_{l=1}^{L-1}\Vert K_{\sigma_l}\Vert \alpha(f_l)}{\prod_{l=1}^L\vert\det W_l^*W_l\vert^{1/4}}\bigg)^{2r}\\
      &\qquad\qquad\qquad\qquad\qquad\qquad + \bigg(\frac{\Vert g\circ v\Vert\prod_{l=1}^{L-1}\Vert K_{\sigma_l}\Vert \alpha(f_l)}{\prod_{l=1}^L\vert\det W_l^*W_l\vert^{1/4}}\bigg)^2\bigg)^{1/2}.
\end{align*}
\end{theorem}

\begin{theorem}\label{thm:nonlinear_linear_koopman}
Assume $\DD$ is in the form~\eqref{eq:D_nonlinear_liner}.
Let 
$$\Dbb_{x_n}=\bigotimes_{i=0}^r1/(i!)(\tilde{\DD}^{\otimes i})^*(P^i)^*g^{(i)}(0)^*p_{x_n}\oplus (\tilde{\DD}^{\otimes r})^*(P^r)^*S(0,\tilde{\DD}u_{\theta}(\cdot))^*p_{x_n}(\cdot).$$
Assume there exists $\epsilon>0$ such that $\Vert S(0,\tilde{\mcl{D}}u_{\theta})\Vert\le \epsilon$.
Assume, in addition, there exists $F>0$ such that for any $n=1,\ldots,N$, $\Vert\Dbb_{x_n}\Vert\le F$.
Then, we have 
\begin{align*}
&\hat{R}(\mcl{V}_{\Theta},x_1,\ldots,x_N)\\
&\le\frac{F}{\sqrt{N}}\sup_{\theta\in{\Theta}}
      \bigg(\sum_{i=0}^r\bigg(\frac{\Vert v\Vert\prod_{l=1}^{L-1}\Vert K_{\sigma_l}\Vert \alpha(f_l)}{\prod_{l=1}^L\vert\det W_l^*W_l\vert^{1/4}}\bigg)^{2i} +\bigg(\frac{\Vert v\Vert\prod_{l=1}^{L-1}\Vert K_{\sigma_l}\Vert \alpha(f_l)}{\prod_{l=1}^L\vert\det W_l^*W_l\vert^{1/4}}\bigg)^{2r}\bigg)^{1/2}.
\end{align*}
\end{theorem}

Since the determinants of weight matrices appear in the denominator of the bound, the bound implies that similar to the standard neural networks, for neural networks with differential operators, high-rank weight matrices such as unitary matrices can also make the generalization error small.
Regarding the factor $\Vert K_{\sigma_l}\Vert$, we can apply Lemmas~\ref{lem:koopman_bounded} and \ref{lem:sigmoid_tanh} to obtain an upper bound.

\begin{remark}
Let $p_x$ be the same function as in Example~\ref{ex:bump_function}.
Then, we have $\Vert p_{x_n}\Vert\to\infty$ as $c\to\infty$.
Thus, the constant $F$ depends on $c$ and $F\to\infty$ as $c\to\infty$.
Note that $F$ does not depend on the weight matrices (see Remark~\ref{rmk:Dx_bound}).
Thus, we can expect that the above bounds also captures the behavior of the standard PINN.
We numerically confirm the applicability of our results to the standard PINNs in Section~\ref{sec:numexp}.
\end{remark}
\begin{remark}\label{rmk:PINN}
Since the variational model $V_{\theta}$ goes to the standard PINN as $c\to\infty$, the Rademacher complexity bound diverges as the model goes to the standard PINNs.
This implies the advantage of VPINNs over the standard PINNs in the sense of the generalization.
We confirm this in Subsection~\ref{subsec:ns}.
\end{remark}

\section{Numerical results}\label{sec:numexp}
We numerically confirm the validity of our bound in this section.
In Subsection~\ref{subsec:ns}, we consider a regularization based on our bound and show its effectiveness.
In Subsection~\ref{subsec:monge}, we numerically investigate the impact of the nonlinearity $r$ on our bound.
All the experiments in this paper were executed with Python 3.10 on an Intel(R) Core(TM) i9-10885H
2.4GHz processor with the Windows 11 operating system.

\subsection{Navier--Stokes equation}\label{subsec:ns}
We consider the 2D steady incompressible Navier--Stokes equations on the unit square domain $\mcl{X} = [0,1]^2$, with the Reynolds number number $\rm{Re} = 100$. 
\begin{align*}
(u\cdot\nabla) u+\nabla p-\frac{1}{\mathrm{Re}}\Delta u&=0,\qquad \nabla u=0, 
\end{align*}
where $u\in H^2(\mcl{X})$ denotes the velocity field and $p\in H^1(\mcl{X})$ denotes the pressure.
The boundary condition is $u(x)  = [1,0]$ for $x\in\{[x_1,1]\,\mid\,x_1\in [0,1]\}$ and $u(x)=[0,0]$ otherwise.
The velocity--pressure pair $[u,p]$ is approximated by a single
fully connected network $u_\theta : \mathbb{R}^{2} \to \mathbb{R}^{3}$ defined as $u_{\theta}(x)=W_4\sigma_3(W_3\sigma_2(W_2\sigma_1(W_1x+b_1)+b_2)+b_3)$, where $W_1\in\RR^{64\times 2}$, $W_2,W_3\in\RR^{64\times 64}$, and $W_4\in\RR^{3\times 64}$.
We set $\sigma_1,\sigma_2,\sigma_3$ as the pointwise hyperbolic tangent activation function.
The two-dimensional input is normalized to $[-1,1]^{2}$ via
$\tilde{{x}} = 2{x} - {1}$.
Weights are initialized with the Glorot uniform scheme~\cite{glorot2010understanding}.
The weak residual is
approximated by midpoint-rule quadrature on a $25 \times 25$ grid
over $\mcl{X}$.
Quadrature nodes are placed at cell centers to avoid boundary points.
We set $p_{x_n}$ as in Example~\ref{ex:bump_function} with $c=0.1$.
The total training objective is
\begin{equation*}
    \mathcal{L}
    = \mathcal{L}_{\mathrm{vpinn}}
    + \lambda_{\mathrm{bc}}\,\mathcal{L}_{\mathrm{BC}}
    + \lambda_{p}\,\mathcal{L}_{p}
    + \lambda_{K}\,\mathcal{L}_{K},
    \label{eq:total_loss}
\end{equation*}
where each component is defined as follows.
The VPINN residual $\mathcal{L}_{\mathrm{vpinn}}$ is set as Eq.~\eqref{eq:vpinn_res} with $\mcl{L}_n(z)=z^2$.
Regarding the boundary-condition loss $\mathcal{L}_{\mathrm{BC}}$,
boundary conditions are enforced via a penalty term:
\begin{equation*}
    \mathcal{L}_{\mathrm{BC}}
    = \frac{1}{N_{\opn{BC}}}\sum_{i=1}^{N_{\opn{BC}}}
      \Vert u_{\theta}({x_i}) - u_{b,i}\Vert^{2},
    \qquad \lambda_{\mathrm{bc}} = 0.1,
\end{equation*}
where $u_{b,i}$ represents the boundary condition at $x_i\in\partial \mcl{X}$.
We used $N_{\opn{BC}}=240$ points on $\partial\mcl{X}$.
About the pressure reference $\mathcal{L}_{p}$,
since the pressure is determined only up to an additive constant in
incompressible flow, we impose the soft constraint $p(0.5,\,0.5)=0$ by setting
\begin{equation*}
    \mathcal{L}_{p}
    = p(0.5,\,0.5)^{2},
    \qquad \lambda_{p} = 0.1.
\end{equation*}
We set the number $N$ of training points as $N=100$.
We consider a regularization based on the bound in Theorem~\ref{thm:poly_koopman}. 
We computed $\tilde{A}_l$ and $D_l$ based on Appendix~\ref{ap:regularization}.
We added a regularization term $0.01\tilde{A}_l/D_l^{1/2}$ to make the factor $\Vert K_{\sigma_l}\Vert/\vert \det W_l^*W_l\vert^{1/4}$ in the bound in Theorem~\ref{thm:poly_koopman} small.
We compared the test loss $\mcl{L}_{\opn{vPINN}}$ with and without the regularization.
We used the Adam optimizer with learning rate $10^{-3}$.
The results are illustrated in Figure~\ref{fig:ns} (a).
We can see that with the regularization, the test loss becomes smaller than without the regularization, which implies the validity of our bound.

To confirm the applicability of our bound to the standard PINN (the case of $c\to\infty$), we also applied the above regularization to the PINN residual $\mcl{L}_{\opn{PINN}}:=1/N\sum_{n=1}^N\mcl{L}_n(\mcl{D}(u_{\theta})(x_n))$.
The results are illustrated in Figure~\ref{fig:ns} (b).
We can see that even for the standard PINN, with the regularization, the test loss becomes smaller than without the regularization, which shows that our bound is also valid for the standard PINNs as we discussed in Remark~\ref{rmk:PINN}.

\begin{figure}[t]
\begin{center}
\subfigure[]{\includegraphics[scale=0.3]{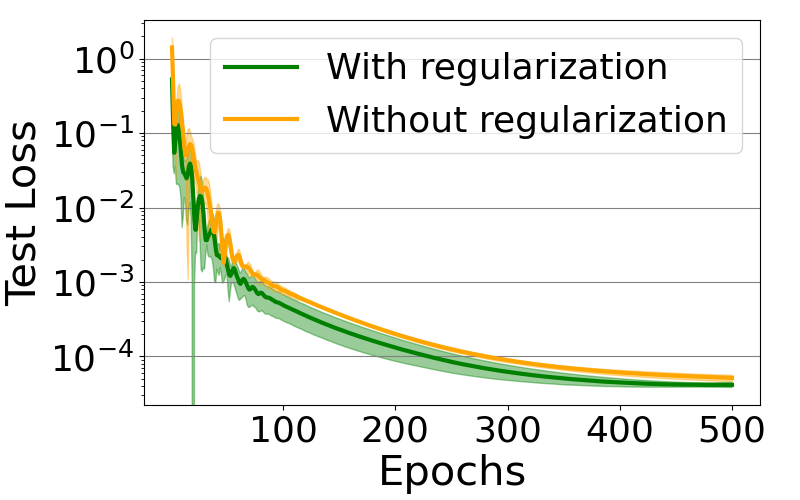}} 
\subfigure[]{\includegraphics[scale=0.3]{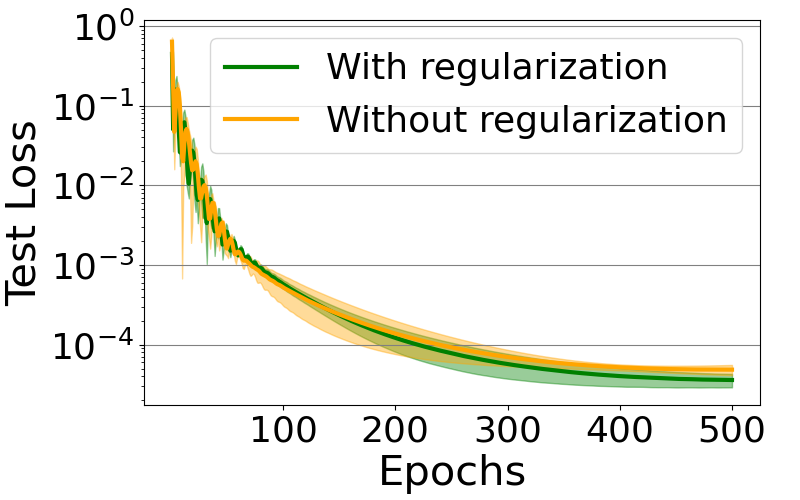}} 
\end{center}\vspace{-.4cm}
\caption{(a) Test loss $\mcl{L}_{\opn{VPINN}}$ with and without the regularization based on Therem~\ref{thm:poly_koopman} (Average $\pm$ standard deviation of three independent runs). (b) Test loss $\mcl{L}_{\opn{PINN}}$ with and without the regularization based on Therem~\ref{thm:poly_koopman} (Average $\pm$ standard deviation of three independent runs).}\label{fig:ns}
\end{figure}

\subsection{Monge-Amp\`{e}re equation}\label{subsec:monge}
We consider the 2D parabolic Monge--Amp\`{e}re
(PMA) equation on $\mcl{X}=[0,1]^2\times [0,1]$ studied in \cite{Gong2020}, which is defined as
\begin{equation*}
    -u_t + \log\det D^2 u = f,
\end{equation*}
where $D^2 u$ denotes the spatial Hessian of $u$. 
The boundary condition is $u(x,t) = [1,0]$ for $x\in\{[x_1,1]\,\mid\,x_1\in [0,1]\}$, $t\in [0,1]$ and $u(x,t)=[0,0]$ otherwise.
We set the same architecture for the neural network as Subsection~\ref{subsec:ns}, but the input dimension is $3$ since the equation in this case depends also on time.
We used the same loss as in Subsection~\ref{subsec:ns} but we do not need the loss $\mcl{L}_p$ in this case.
We set $f(x,y)=1 + \log 2 + 2x$ and consider the same regularization based on our bound.
We set the number $N$ of training points as $N=100$.
We used $N_{\opn{BC}}=720$ points on $\partial\mcl{X}$ to compute $\mcl{L}_{\opn{BC}}$ for the boundary condition.
We used the Adam optimizer with learning rate $10^{-3}$.
Figure~\ref{fig:scatter} shows the scatter plot of the test error versus  $\tilde{A}_l/D_l^{1/2}$.
Since $\DD$ is not polynomial in this case, $A_l$ in Theorem~\ref{thm:nonlinear_linear_koopman} can become larger.
We can see that the relationship between the test error and  $\tilde{A_l}/D_l^{1/2}$ is not linear, and described by a high-order polynomial.
Indeed, Table ~\ref{tab:corr} shows the correlation coefficient between the test error and  $(\tilde{A}_l/D_l^{1/2})^r$.
We can see that the correlation coefficient becomes larger as $r$ becomes larger, which corresponds to the fact that the differential operator is described by the log function.

\begin{figure}[t]
\begin{minipage}[t]{0.5\linewidth}
    \centering
    \includegraphics[scale=0.2]{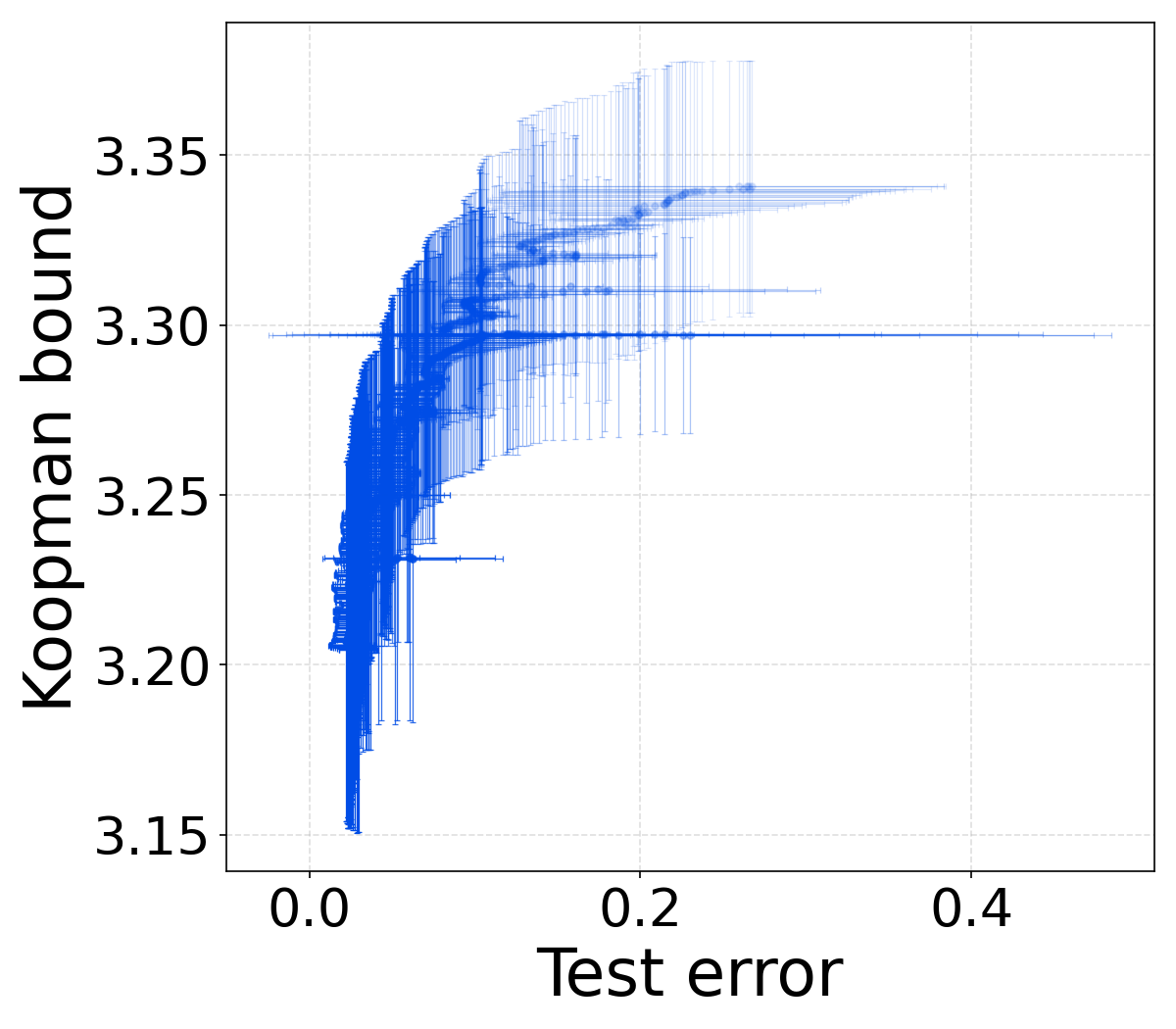}\vspace{-.3cm}
    \caption{Scatter plot of the test error versus  $\tilde{A}_l/D_l^{1/2}$ (Average of 3 independent runs).}
    \label{fig:scatter}
\end{minipage}\quad
\begin{minipage}[t]{0.48\linewidth}
    \centering
    \vspace{-2.5cm}
    \captionof{table}{: Correlation coefficient between the test error and  $(\tilde{A}/D_l^{1/2})^r$.}
    \label{tab:corr}
    \begin{tabular}{c|c}
        $r$ &  correlation coefficient\\
        \hline
         1& 0.8794\\
         2 & 0.8842\\
         3 & 0.8882
    \end{tabular}
    
\end{minipage}
\end{figure}


\section{Conclusion and limitation}\label{sec:conclusion}
In this paper, we derived a new generalization bound for neural networks that involve differential operators with respect to input variables.
The proposed bound allows us to apply the Koopman-based approach, which shows that high-rank networks can generalize well even in settings involving differential operators.
It also shows that the nonlinearity of the differential operator exponentially enlarges the bound.

Our bound is a uniform bound that does not take the learning algorithm into account.
Deriving algorithm-dependent bounds using Koopman-based approach is left for future work. 

\section*{Acknowledgment}
This work was supported by JST ASPIRE JPMJAP2329.

\bibliography{koopmanbib}
\bibliographystyle{plain}

\clearpage
\appendix
\section*{Appendix}

\section{Proofs}\label{ap:proofs}

\begin{proof}[Proof of Theorem~\ref{thm:linear}]
By the Cauchy--Schwartz inequality and the Jensen's inequality, we have 
\begin{align*}
  &\EE\bigg[\sup_{u_{\theta}\in\mcl{U}_{\Theta}}\frac{1}{N}\sum_{n=1}^{N}
      \langle p_{x_n},\DD(u_{\theta})\rangle\,\epsilon_n\bigg]
  \le \frac1N\EE\bigg[\sup_{u_{\theta}\in\mcl{U}_{\Theta}}\|u_{\theta}\|\,\bigg\|\sum_{n=1}^{N} \mathcal{D}^* p_{x_n}\epsilon_n\bigg\|
      \bigg]\\
  &= \frac{1}{N}\sup_{u_{\theta}\in\mcl{U}_{\Theta}}
      \|u_{\theta}\|\EE\bigg[
        \bigg(\sum_{n,m=1}^{N}
          \langle\Dbb_{x_n}\epsilon_n,\,\Dbb_{x_m}\epsilon_m\rangle
        \bigg)^{1/2}\bigg]
      \\
  &= \frac{1}{N}\sup_{u_{\theta}\in\mcl{U}_{\Theta}}
      \|u_{\theta}\|\EE\bigg(\bigg[
        \sum_{n,m=1}^{N}
          \langle\Dbb_{x_n}\epsilon_n,\,\Dbb_{x_m}\epsilon_m\rangle
        \bigg]\bigg)^{1/2}
      \\
  &= \frac{1}{N}\sup_{u_{\theta}\in\mcl{U}_{\Theta}}
      \|u_{\theta}\|\bigg(
        \sum_{n,m=1}^{N}
          \langle\Dbb_{x_n},\,\Dbb_{x_m}\rangle\,
          \EE[\epsilon_n\epsilon_m]
        \bigg)^{1/2} \\
  &= \frac{1}{N}\sup_{u_{\theta}\in\mcl{U}_{\Theta}}
      \|u_{\theta}\|\bigg(
        \sum_{n=1}^{N}
          \langle\Dbb_{x_n},\,\Dbb_{x_n}\rangle
        \bigg)^{1/2} 
    \le \frac{F}{\sqrt{N}}\sup_{u_{\theta}\in\mcl{U}_{\Theta}}
      \|u_{\theta}\|.
\end{align*}
\end{proof}

\section{Regularization based on the proposed bound}\label{ap:regularization}
We consider a regularization based on the bounds in theorems in Subsection~\ref{subsec:koopman}. 
We set the space of pre-activation vectors as $\tilde{\mcl{X}}_l=\Vert W_l\Vert [-1,1]^{d_{l-1}}+\Vert b_l\Vert_{\infty}$ for $l=1,\ldots,4$, where $d_0=d$ and $d_l$ is the output dimension of the $l$th layer for $l=1,\ldots,3$.
According to Lemma~\ref{lem:sigmoid_tanh}, we computed the upper bound $A_l:=\sup_{x\in {\sigma_l(\tilde{\mcl{X}_j})}} \vert J\sigma^{-1}(x)\vert$ of $\Vert K_{\sigma_l}\Vert^2$.
To avoid that the upper bound becomes too large, we computed $\tilde{A}_l:=1/(1+1/A_l)$ instead of $A_l$.
The factor $\tilde{A}_l$ increases monotonically  as $A_l$ increases but is bounded.
Regarding the factor $(\det W_l^*W_l)^{1/4}$, since $(\det W_l^*W_l)^{1/2}$ is described by the product of the singular values of $W_l$, we consider the singular values of $W_l$.
Let $\sigma_1 \geq \cdots \geq \sigma_R > 0$ be the singular values
of $W_l$, where $R$ is the rank of $W_l$.  
The geometric mean
\begin{align*}
    D_l
    = \Bigl(\prod_{k=1}^{R}\sigma_k\Bigr)^{1/R}
    = \vert \det W_l\bigr\vert^{1/R}
    = \exp\!\bigg(\frac{1}{R}\sum_{k=1}^{R}\ln\sigma_k\bigg)
    \label{eq:Dl}
\end{align*}
is computed in log-space for numerical stability with an offset
$10^{-8}$ before taking logarithms.
By adding a regularization term $0.01\tilde{A}_l/D_l^{1/2}$, we can make the factor $\Vert K_{\sigma_l}\Vert/\vert \det W_l^*W_l\vert^{1/4}$ in the bound in Theorem~\ref{thm:poly_koopman} small.

\end{document}